\documentclass{article}
\usepackage{ijcai20}

\usepackage{times}
\usepackage{soul}
\usepackage{url}
\usepackage[hidelinks]{hyperref}
\usepackage[utf8]{inputenc}
\usepackage[small]{caption}
\usepackage{graphicx}
\usepackage{amsmath}
\usepackage{amsthm}
\usepackage{booktabs}
\usepackage{algorithm}
\usepackage{algorithmic}
\usepackage{amssymb}
\usepackage{color}
\usepackage{epsfig}
\usepackage{latexsym}
\usepackage{algorithm}
\usepackage{algorithmic}
\usepackage{bbm}
\usepackage{caption,subcaption}
\usepackage{pgffor}
\usepackage{pifont} 
\usepackage{balance}
\usepackage{multirow}
\usepackage{dblfloatfix}
\usepackage{enumitem}

\usepackage[usenames,svgnames,xcdraw,table]{xcolor}
\definecolor{DarkBlue}{rgb}{0.1,0.1,0.5}

\usepackage{flushend}

\hypersetup{colorlinks=true,
linkcolor=DarkBlue,
citecolor=DarkBlue,
urlcolor=black
}

\newtheorem{theorem}{Theorem}
\newtheorem{corollary}{Corollary}[theorem]

\title{Ensuring Fairness under Prior Probability Shifts}

\author{
Arpita Biswas$^1$\and
Suvam Mukherjee$^2$\\
\affiliations
$^1$Indian Institute of Science\\
$^2$Microsoft Research\\
\emails
arpitab@iisc.ac.in,
suvamm@outlook.com
}

\newcommand{\cape}{\mathtt{CAPE}}

\newcommand{\BClass}{\mathtt{Max\_Acc}}

\newcommand{\Testdata}{\mathbb{D}}

\newcommand{\ppsampling}{\mathtt{PP}\text{-}\mathtt{SAMPLING}}
\newcommand{\qalgo}{\mathtt{Q}\text{-}\mathtt{ALG}}
\newcommand{\clalgo}{\mathtt{C}\text{-}\mathtt{ALG}}

\newcommand{\predPrev}{\hat{\rho}}
\newcommand{\truePrev}{{\rho}}
\newcommand{\estPrev}{\hat{q}}

\DeclareMathOperator*{\argmin}{arg\,min}

\newlength\myindent
\newcommand\bindent{%
  \begingroup
  \setlength{\itemindent}{\myindent}
  \addtolength{\algorithmicindent}{\myindent}
}
\newcommand\eindent{\endgroup}                                                                                                                                                                                                                                                                                                                                                                                                                                                                                                                                                                                                           

\begin{document}

\maketitle

\begin{abstract}
In this paper, we study the problem of fair classification in the presence of prior probability shifts, where the training set distribution differs from the test set. This phenomenon can be observed in the yearly records of several real-world datasets, such as recidivism records and medical expenditure surveys. If unaccounted for, such shifts can cause the predictions of a classifier to become unfair towards specific population subgroups. While the fairness notion called \emph{Proportional Equality} (PE) accounts for such shifts, a procedure to ensure PE-fairness was unknown. 

In this work, we propose a method, called $\cape$, which provides a comprehensive solution to the aforementioned problem. $\cape$ makes novel use of prevalence estimation (\emph{quantification}) techniques, sampling and an ensemble of classifiers to ensure fair predictions under prior probability shifts. We introduce a metric, called \emph{prevalence difference} (PD), which $\cape$ attempts to minimize in order to ensure PE-fairness. We theoretically establish that this metric exhibits several desirable properties. 

We evaluate the efficacy of $\cape$ via a thorough empirical evaluation on synthetic datasets. We also compare the performance of $\cape$ with several popular fair classifiers on real-world datasets like COMPAS (criminal risk assessment) and MEPS (medical expenditure panel survey). The results indicate that $\cape$ ensures PE-fair predictions, while performing well on other performance metrics.
\end{abstract}

\section{Introduction}
\label{sec:intro}
Machine learning techniques are being increasingly applied in making important societal decisions, such as criminal risk assessment, school admission, hiring, sanctioning of loans, etc. Given the impact and sensitivity of such predictions, there is warranted concern regarding implicit discriminatory traits exhibited by such techniques. Such discrimination may be detrimental for certain population subgroups with a specific race, gender, ethnicity, etc, and may even be illegal under certain circumstances~\cite{angwin2016machine}. These concerns have spurred vast research in the area of \emph{algorithmic fairness} \cite{corbett2018measure,dressel2018accuracy,chouldechova2018frontiers,friedler2018comparative,zhang2018achieving,berk2017fairness,kleinberg2018discrimination,barocas2016big,chouldechova2017fair,romei2014multidisciplinary}. Most of these papers aim to establish fairness notions for a group of individuals (differentiated by their race, gender, etc.), and are classified as \textit{group fairness} notions.

A possible, and less studied, cause for unfairness in predictions involve distributional changes (or drift) between the training and test datasets. Disparities can be introduced when the sub-populations evolve differently over time~\cite{barocas2017fairness}. There are important real-world scenarios where a type of distributional change, called \textit{prior probability shift}, occurs. Informally, a prior probability shift occurs when the fraction of positively labeled instances differ between the training and the test datasets (see Section~\ref{sec:shift} for a formal definition). A concrete example is the COMPAS 
 dataset~\cite{url:compas} which contains demographic information and criminal history of defendants, and records whether they recommitted a crime within a certain period of time (positive labels are given to the re-offenders, while others have negative labels). We observe that, among the valid records screened in the year $2013$, the fraction of Caucasian and African-American re-offenders were $0.327$ and $0.486$, respectively. However, in $2014$, these fractions were $0.636$ and $0.706$, respectively. This indicates that the extent of prior probability shift differs among Caucasian and African-American defendants, between the records of $2013$ and $2014$. 

If such distributional changes are unaccounted for, a classifier may end up being unfair towards the population subgroups which exhibit prior probability shifts. For example, if the rate of recidivism among a particular sensitive group reduces drastically, then a classifier trained with a higher rate of recidivism can create extreme unfairness towards individuals of that sub-population. In this work, we address these concerns and propose a method to obtain fair predictions under prior probability shifts.\\ 

\noindent \textbf{Related Work. }A large body of work defines various \textit{group fairness} notions and provides algorithms to mitigate unfairness. Among these, \emph{Proportional Equality} (PE) \cite{biswas2019fairness,hunter2000proportional}  appears to be the most appropriate fairness notion for addressing prior probability shifts among population subgroups (see Section~\ref{sec:pe} for definition). However, the existing results stop short of providing a procedure to ensure PE-fair predictions. We address this concern by proposing an end-to-end solution.


Apart from PE, there are other group fairness notions, none of which address prior probability shifts explicitly, such as \textit{Disparate Impact}~\cite{feldman2015certifying,zafar2017afairness,kamiran2012data,calders2009building}, \textit{Statistical Parity}~\cite{corbett2017algorithmic,kamishima2012fairness,zemel2013learning},
\textit{Equalized Odds}~ \cite{hardt2016equality,kleinberg2017inherent,woodworth2017learning}, and \textit{Disparate Mistreatment}~\cite{zafar2017bfairness}. 
%
%

Unfortunately, all these fairness constraints are often non-convex, thereby making the optimization problem (maximizing accuracy subject to fairness constraints) difficult to solve efficiently. Several papers provide convex surrogates of the non-convex constraints~\cite{goh2016satisfying,zafar2017afairness}, or finds near-optimal near-feasible solutions~\cite{cotter2018training,celis2018classification}, or propose techniques to reduce dependence of group information on the predictions~\cite{kamiran2012data,kamiran2012decision,pleiss2017fairness,zhang2018mitigating}\footnote{Note that the \textit{group fairness} notions require the test set to be of (statistically) significant size for fairness evaluation.}. However, most of these solutions assume that the training and test datasets are identically and independently drawn from some common population distribution, and thus suffer in the presence of prior probability shifts (we provide empirical evidences in Section~\ref{sec:eval}).\\

\noindent \textbf{Our Contributions. }To the best of our knowledge, we are the first to propose an end-to-end solution to ensure fair predictions in the presence of prior probability shifts. 
\begin{enumerate}[leftmargin=*]
    \item We design a system called $\cape$ (\underline{C}ombinatorial \underline{A}lgorithm for \underline{P}roportional \underline{E}quality)
    in Section~\ref{sec:cape}.
    \item We introduce a metric called \emph{Prevalence Difference} (PD), which $\cape$ attempts to minimize in order to ensure PE-fairness. We theoretically establish that the PD metric exhibits several desirable properties (Theorems~\ref{thm:perfect}, \ref{thm:app})---in particular, we show that maximizing the accuracy of any subgroup is not at odds with minimizing PD. This metric also provides insights into why the predictions of $\cape$ are fair (Theorem~\ref{claim:cape}). We discuss these in Section~\ref{sec:pd-fair} and \ref{sec:theory}.
    \item We perform a thorough evaluation of $\cape$ on synthetic and real-world datasets, and compare with several other fair classifiers. In Section~\ref{sec:eval}, we provide empirical evidence that $\cape$ provides PE-fair predictions, while performing well on other fairness metrics.
\end{enumerate}
\section{Background and Notations}
\label{sec:prelims}
In this paper, we focus on the binary classification problem, under \textit{prior probability shifts}. Let $\hat{h}:\mathcal{X}\mapsto\mathcal{Y}$ be the prediction function, defined in some hypothesis space $\mathcal{H}$, where $\mathcal{X}\subset \mathbb{R}^m$ is the $m$-dimensional feature space and $\mathcal{Y} = \{0,1\}$ is the label space. The goal of a classification problem is to learn the function $\hat{h}$ which minimizes a target loss function, say, misclassification error $\mathbb{P}[\hat{h}(X)\neq Y]$ (variables $X$ and $Y$ denote feature vectors and labels). However, if these predictions $\hat{h}(\cdot)$ are used for societal decision making, it becomes crucial to ensure lower misclassification error not only on an average but also within each group defined by their sensitive attribute values such as race, gender, ethnicity, etc. Dropping these sensitive attributes blindly from the dataset may not be enough to alleviate discrimination since some non-sensitive features can be closely correlated to the sensitive attributes~\cite{zliobaite2016using,corbett2018measure,hardt2016equality}. 
Hence, most existing solutions assume access to the sensitive attributes. 
In the presence of such a sensitive attribute with $G$ sub-populations, the goal is to learn $\hat{h}:\mathcal{X}\times [G] \mapsto\mathcal{Y}$ satisfying certain group-fairness criteria (where $[G]$ denotes the set $\{0,1,\ldots,G-1\}$). We use variable $Z\in[G]$ to denote group membership (one can encode multiple sensitive attributes into $[G]$).  
We assume that the \emph{training dataset} $D=\{(x_i,z_i, y_i)_{i=1}^N\}$ is drawn from an unknown joint distribution $\mathcal{P}$ over $\mathcal{X}\times[G]\times\mathcal{Y}$. The performance of the classifier is measured using a new set of data, referred as \emph{test dataset} $\Testdata=\{(x_j, z_j, y_j)_{j=1}^n\}$, by observing how accurate and fair the $\hat{h}(x_j)$s are with respect to the true labels $y_j$s. 

Next, we focus on an important phenomenon called \textit{prior probability shift}, which may cause a learned classifier to be unfair in its predictions on a test dataset.

\subsection{Prior Probability Shift}
\label{sec:shift}
\textit{Prior probability shift}~\cite{saerens2002adjusting,moreno2012unifying,kull2014patterns} occurs when the prior class-probability $\mathcal{P}(Y)$ changes between the training and test sets, but the class conditional probability $\mathcal{P}(X|Y)$ remains unaltered. Such changes, within a sub-population, occur in many real-world scenarios, that is, $\mathcal{P}(X|Y$=$1, Z$=$z)$ remains constant but $\mathcal{P}(Y$=$1|Z$=$z)$ 
changes between the training and test datasets. If left unaccounted for, it may lead to \textit{unfair} predictions~\cite{barocas2017fairness}. 

\subsection{Proportional Equality}
\label{sec:pe}
To address the fairness concern under prior probability shifts, a notion called \textit{proportional equality} (PE) was formalized in ~\cite{biswas2019fairness}. A classifier is said to be PE-fair if it has low values for the following expression:
\[\mbox{PE}^{z,z'}:=\left|\frac{\rho^z_{\Testdata}}{\rho^{z'}_{\Testdata}} - \frac{\hat{\rho}^z_{\Testdata}}{\hat{\rho}^{z'}_{\Testdata}}\right| \mbox{ for all } z,z'\in[G]\]
\noindent $\bullet$ \textit{True prevalence} $\rho^z_{\Testdata}$ is the fraction of population, from the group $z$, labeled positive in the dataset ${\Testdata}$.
\begin{equation}\label{eq:TruePrev}
{\rho}_{\Testdata}^z:= \frac{\left|\{(x_i,z_i,y_i)\in\Testdata\ \vert\ y_i=1, z_i=z\}\right|}{\left|\{(x_i,z_i, y_i)\in\Testdata\ \vert\ z_i=z\}\right|}. 
\end{equation}
\noindent  $\bullet$ \textit{Prediction prevalence} $\hat{\rho}^z_{\Testdata}$ is the fraction of population, from the group $z$, predicted positive by the classifier for $\Testdata$.
\begin{equation}\label{eq:PredPrev}
\hat{\rho}_{\Testdata}^z:= \frac{\left|\{(x_i,z_i,y_i)\in\Testdata\ \vert\ \hat{y}_i=1, z_i=z\}\right|}{\left|\{(x_i,z_i, y_i)\in\Testdata\ \vert\ z_i=z\}\right|}.
\end{equation}

However, Biswas and Mukherjee~\shortcite{biswas2019fairness} do not provide any algorithm for ensuring PE-fair predictions. Any such algorithm must deal with the following key challenges:
\begin{enumerate}[leftmargin=*]
 \item PE$^{z,z'}\leq\epsilon$ (for a small $\epsilon$) is a non-convex constraint. Thus, it is hard to directly optimize for accuracy subject to this constraint for all $z,z'\in[G]$. 
    \item The definition of PE uses true prevalences of the test datasets $\rho_\Testdata^z$, which are unavailable to the classifier during the prediction phase. Thus, an algorithm needs to \emph{estimate} these prevalences. Techniques from the \textit{quantification} literature can be leveraged to solve this concern, which we describe next.
\end{enumerate}

\subsection{Quantification Problem}
\label{sec:quant}
Quantification learning (or \emph{prevalence estimation}) is a supervised learning problem, introduced by Forman~\shortcite{forman2005counting}. It aims to predict an aggregated quantity for a set of instances. 
The goal is to learn a function, called \emph{quantifier} $q:\mathcal{X}^{\mathbb{N}}\mapsto [0,1]$, that outputs an estimate of the true prevalence of a finite, non-empty and unlabeled test set $\Testdata\sim\mathcal{X}^{\mathbb{N}}$. As highlighted by Forman, quantification is not a by-product of classification~\cite{gonzalez2017quantification}. In fact, unlike assumptions made in classification, quantification techniques account for changes in prior probabilities $\mathcal{P}(Y|Z)$ within subgroups, while assuming $\mathcal{P}(X|Y,Z)$ remain the same over the training and test datasets. This allows quantifiers to perform better than na\"ive classify and count techniques, as demonstrated by Forman~\shortcite{forman2006quantifying}.

Some commonly used algorithms to construct quantifiers are \emph{Adjusted Classify and Count} (\texttt{ACC}) \cite{forman2006quantifying}, \emph{Scaled Probability Average} (\texttt{SPA})~\cite{bella2010quantification}, and \texttt{HDy}~\cite{gonzalez2013class}. These algorithms can be used to estimate the prevalence of a group in the test set.

For ease of exposition, we describe a simple quantification technique, \texttt{ACC}. This method learns a binary classifier from the training set and estimates its true positive rates ($\mathit{TPR}$) and false positive rates ($\mathit{FPR}$) via $k$-fold cross-validation. Using this trained model, the algorithm counts the number of cases on which the classifier outputs positive on the test set. Finally, the true fraction of positives (true prevalence) is estimated via the equation 
$p = \frac{p'-\mathit{FPR}}{\mathit{TPR}-\mathit{FPR}}$,
where $p'$ denotes the fraction of predicted positives, $p':= \frac{\#\mathit{predicted\_ positives}}{\#\mathit{test\_data\_ points}}$.  The use of $\mathit{TPR}=\frac{TP}{TP+FN}$ and $\mathit{FPR}=\frac{FP}{FP+TN}$ from the training set can be justified by the assumption that $P(X|Y)$ remains same in the training and test datasets. This simple algorithm turns out to provide good estimates of prevalences under prior probability shifts. However, for our experiments, we use \texttt{SPA}~\cite{bella2010quantification}, which uses a probability estimator instead of a classifier, and turns out to be more robust to variations while estimating probabilities of a dataset with a few samples.\\ 

Next, we discuss $\cape$, which provides a comprehensive solution to the above problems by combining quantification techniques along with training an ensemble of classifiers.

\section{CAPE}
\label{sec:cape}

In this section, we introduce $\cape$~(\underline{C}ombinatorial \underline{A}lgorithm for \underline{P}roportional \underline{E}quality), for ensuring PE-fair predictions. 
$\cape$ takes as input a training dataset $D$ and a vector of desired prevalences $\Theta=(\theta_1, \ldots, \theta_k)\in[0,1]^k$. $\cape$ is separately trained for each group $z \in [G]$, since we hypothesize that the relationship between the non-sensitive features $X$ and the outcome variable $Y$ may differ across groups. Thus, each group would be best served by training classifiers on datasets obtained from the corresponding group\footnote{Training a separate classifier for a small-sized subgroup may be inappropriate. For the datasets we consider, this issue never arises.}. Such decoupled classifiers are also considered by Dwork et al.~\shortcite{dwork2017decoupled}, but they do not handle prior probability shifts. 

The training phase outputs, for each group $z$, the following:
\begin{enumerate}[leftmargin=*]
\item a set of $|\Theta|$ classifiers, each trained using a sampling of the training dataset obtained by the module $\ppsampling$, which takes as input a prevalence parameter $\theta\in\Theta$ and a training set with $N_z$ data points. It randomly selects, with replacement, $\theta \times N_z$ instances with $Y$=$1$ and $(1-\theta)\times N_z$ instances with $Y$=$0$. Thus, it outputs a sample of size $N_z$. Each classifier is thus specialized in providing accurate predictions on datasets with particular prevalences.
\item a quantifier $\estPrev^z(\cdot)$, generated by the $\qalgo$ module, which is subsequently used in the prediction phase of $\cape$ to estimate the true prevalence of the test dataset, $\truePrev_\Testdata^z$. Separate quantifiers are created for each group since the extent of prior probability shifts may differ across groups. 
\end{enumerate}

During the prediction phase, for each group $z$, an estimate of the prevalence of the test data $\Testdata^z$ is obtained using $\hat{q}^z(\cdot)$ (learned in the training phase). This estimate is then used to choose the classifier $J_z$ that minimizes the \emph{prevalence difference} metric (Section~\ref{sec:pd-fair}). Finally, $\cape$ outputs the predictions of the classifier $J_z$ on the test set $\Testdata^z$. 

\begin{algorithm}[h!]\label{algo:CAPE}
\scriptsize{
\begin{flushleft}
\textbf{\underline{Training Phase}:}\\ 
\textbf{Input:} Training dataset $D$, $\qalgo$, $\ppsampling$, $\clalgo$, and a vector of prevalence parameters $\Theta:= (\theta_1, \ldots, \theta_k\}$.\\
\end{flushleft}
\begin{algorithmic}
	 \STATE \textbf{Step 1:} Partition $D=\{(x_i, z_i, y_i)_{i=1}^N\}$ based on $z_i$ values.
    \bindent
    \STATE $D^z\leftarrow\ \{(x_i, z_i, y_i)\in D\ |\ z_i=z\}$ for each group $z$.
    \eindent
    \STATE \textbf{Step 2:} Create quantifiers, one for each $z$.
    \bindent
    \STATE $\hat{q}^z(\cdot)\leftarrow$ $\qalgo$($D^z$).
    \eindent
    \STATE \textbf{Step 3:} Create a set of $k$ classifiers, for each $z$.
    \bindent
    \FORALL{$\theta$ in $\{\theta_1, \ldots, \theta_k\}$}
    	\STATE $T^z\leftarrow$ $\ppsampling$ ($D^z, \theta$).
    	\STATE $\hat{h}^z_{\theta}(\cdot)\leftarrow \clalgo$ ($T^z$).
    \ENDFOR
    \eindent
    \STATE \textbf{Output:} $\hat{q}^z$ and $(\hat{h}^z_{\theta_j})_{j=1}^k$.
\end{algorithmic}
$\ $
\begin{flushleft}
\textbf{\underline{Prediction Phase}:}\\
\textbf{Input:} Test dataset $\Testdata$, and the quantifiers and classifiers obtained after the training phase.\\
\end{flushleft}
\begin{algorithmic}
	 \STATE \textbf{Step 1:} Partition  $\Testdata=\{(x_i, z_i, y_i)_{i=1}^n\}$ based on $z_i$ values.
    \bindent
    \STATE ${\Testdata}^z\leftarrow\ \{(x_i, z_i, y_i)\in {\Testdata}\ |\ z_i=z\}$ for each group $z$.
    \eindent
    \STATE \textbf{Step 2:} Estimate prevalences $\hat{q}^z({\Testdata}^z)$ using the quantifiers built in training phase.
    \STATE \textbf{Step 3:} Choose the best classifier in terms of estimated PD, for each $z$.
    \bindent
    \FORALL{$\theta$ in $\{\theta_1, \ldots, \theta_k\}$}
    	\STATE $\hat{y}^i_{\theta}\leftarrow \mathtt{sgn}(\hat{h}^z_{\theta}(x_i))$ for all $i\in\{1,\ldots,|\Testdata^z|\}$.
    	\STATE $\predPrev^z_{\theta}\leftarrow{|\{i\in {\Testdata}^z:\ \hat{y}^i_{\theta}==1\}|}/{|\Testdata^z|}$.
    \ENDFOR
    \STATE $J_z \! \leftarrow \! \displaystyle \argmin_{\theta\in\Theta} |\predPrev^z_{\theta} - \hat{q}^z(\Testdata^z)|$ \COMMENT{Best Classifier for $z$}
    \eindent
    \STATE \textbf{Output:} The predictions $\hat{y}^z_{J_z}$ for group $z$.    
\end{algorithmic}
\caption{\label{algo:cape}The $\cape$ meta-algorithm.}}
\end{algorithm}

Note that $\cape$ provides the flexibility to plug in any classification and quantification algorithm into modules $\clalgo$ and $\qalgo$. Key to $\cape$ is the \emph{prevalence difference} metric, used in Step $3$ of the prediction phase. We formalize the metric and discuss some of its properties in the next section.


\subsection{Prevalence Difference Metric}
\label{sec:pd-fair}
We define the \emph{prevalence difference} (PD) metric, for each group $z$, as: $\Delta_{\Testdata}^z:={\left|{\rho}_{\Testdata}^{z} - \hat{\rho}_{\Testdata}^{z}\right|}$, where, $\rho_{\Testdata}^z$ and $\hat{\rho}_{\Testdata}^z$ denote the true and predicted prevalences of the dataset $\Testdata$ (as defined in Equations~\ref{eq:TruePrev} and \ref{eq:PredPrev}, respectively). Hereafter, we drop the subscripts and superscripts on $\Delta$, $\rho$ and $\hat{\rho}$ whenever we refer to the population in aggregate.

Note that the true prevalence $\rho_{\Testdata}^z$ of test set {$\Testdata$} cannot be used during the prediction phase. Thus, replacing $\rho_{\Testdata}^z$ with $\hat{q}^z({\Testdata}^z)$ in the definition of $\Delta_{\Testdata}^z$ provides a measure to choose the best classifier $J^z$ for the group $z$. Also, unlike PD, other performance metrics like accuracy, FPR or FNR are not suitable for choosing the best classifier since these metrics require the \emph{true} labels of the test datasets. We use the PD metric for: (1)~choosing the best classifier in the prediction phase and (2)~measuring the performance of the predictions, since a high value of $\Delta^z$ implies the inability to account for prior probability shift for the group $z$.

The PD metric is somewhat different from the fairness metrics aiming to capture parity between two sub-populations. Such fairness metrics may often require sacrificing the performance on one group to maintain parity with the other group. However PD, in itself, believes that the two groups should be treated differently since each group may have gone through a different change of prior probabilities. A high $\Delta^z$ indicates high extent of harm caused by the predictions made towards to the group $z$. Thus, to audit the impact of a classifier's predictions on a group $z$, it is important to evaluate for $\Delta^z$, along with accuracy, FNR and FPR values within each group.

Next, we show that a \textit{perfect classifier} ($100\%$ accurate) attains zero prevalence difference. Additionally, we show that a classifier with high accuracy on any sub-group also attains a very low $\Delta$ for that subgroup. Empirically, we observe that low $\Delta$ results in PE-fair predictions.

\subsection{Theoretical Guarantees}\label{sec:theory}
We first show a simple result--- a classifier whose predictions are exactly the ground truth also attains $\Delta = 0$, thereby satisfying our proposed metric used for selecting the best classifier. Note that a perfect classifier may \emph{not} satisfy fairness notions such as disparate impact and statistical parity.

\begin{theorem}
A perfect classifier always exhibits $\Delta=0$.
\label{thm:perfect}
\end{theorem}

\begin{proof}
    Let us consider a perfect classifier $C$ whose predictions are equal to the ground truth i.e., $\hat{y}(x)=y(x)$ for all instances $x\in\mathcal{X}$, where $\hat{y}(x)$ is the label predicted by the classifier $C$ for the instance $x$. Thus, for each $z$, the true prevalence $\rho^z$  is equal to the prediction prevalence $\hat{\rho}^z$, according to the definitions in Equations \ref{eq:TruePrev} and \ref{eq:PredPrev}. Thus, the prevalence difference $\Delta^z = |\rho^z - \hat{\rho}^z|=0$.
\end{proof}

\begin{theorem}\label{thm:app}
If the overall accuracy of a classifier $C$ is $(1-\delta)$, where $\delta\in(0,1)$ is a very small number, then the overall \textit{prevalence difference} for $C$ is $\Delta=\delta - 2\min\left\{\frac{\mathtt{FN}}{n}, \frac{\mathtt{FP}}{n}\right\}$, where $\mathtt{FN}$ and $\mathtt{FP}$ denote number of false negatives and false positives respectively in the test dataset with $n$ instances. This further implies that $\Delta\leq \delta$.
\end{theorem}

\begin{proof} 
    Let $(\hat{y}_i)_{i=1}^n$ denote the predictions of a classifier $C$ on a test dataset $\{(x_i, y_i)_{i=1}^n\}$. Some other notations that we use for the proof, are:\\
    $\mathtt{TP}:=\left| \{i: y_i=1\ \&\ \hat{y}_i=1\}\right|$ ($\#$ true positives).\\
    $\mathtt{TN}:={\left| \{i: (y_i=0\ \&\ \hat{y}_i=0\}\right|}$ ($\#$ true negatives).\\
    $\mathtt{FP}:={\left| \{i: y_i=0\ \&\ \hat{y}_i=1\}\right|}$ ($\#$ false positives).\\
    $\mathtt{FN}:={\left| \{i: y_i=1\ \&\ \hat{y}_i=0\}\right|}$ ($\#$ false negatives).
    
    \noindent Note that $\mathtt{TP}+\mathtt{TN}+\mathtt{FP}+\mathtt{FN} = n$. Let $\rho$ and $\hat{\rho}$ be the true and prediction prevalences. Then, the \textit{prevalence difference} can be written as:
    \begin{align}
    \Delta = |\rho - \hat{\rho}| =  \left|\frac{\mathtt{TP}+\mathtt{FN}}{n} - \frac{\mathtt{TP}+\mathtt{FP}}{n}\right| = \frac{|\mathtt{FN}-\mathtt{FP}|}{n} \label{eq:pd}
    \end{align}
    
    Let the accuracy of a classifier on a test dataset be ($1-\delta$) where $\delta\in(0,1)$. Then,
    \begin{align}
    \frac{\mathtt{TP}+\mathtt{TN}}{n}  =  1-\delta \quad
    \Rightarrow  \frac{\mathtt{FN}+\mathtt{FP}}{n} = \delta \label{eq:acc}
    \end{align}
    
    Without loss of generality, let us assume $\mathtt{FN}\geq \mathtt{FP}$. Thus, Equation~\ref{eq:acc} can be written as:
    \begin{align}
    \frac{\mathtt{FN}-\mathtt{FP} + 2\mathtt{FP}}{n} = \delta \quad \Rightarrow \frac{\mathtt{FN}-\mathtt{FP}}{n}  =  \delta - \frac{2\mathtt{FP}}{n} \label{eq:fn-fp}
    \end{align}
    Similarly, assuming $\mathtt{FP}\geq \mathtt{FN}$ we obtain
    \begin{equation}
    \frac{\mathtt{FP}-\mathtt{FN}}{n}  =  \delta - \frac{2\mathtt{FN}}{n}  \label{eq:fp-fn}
    \end{equation}
    
    Combining Equation \ref{eq:pd}, \ref{eq:fn-fp} and \ref{eq:fp-fn}, we get the following:
    \begin{align}
    &\Delta = \frac{|\mathtt{FP}-\mathtt{FN}|}{n}  =  \delta - 2\min\left\{\frac{\mathtt{FN}}{n}, \frac{\mathtt{FP}}{n}\right\}\nonumber\\
    \Rightarrow &\Delta\leq \delta. 
    \end{align}
    Thus, when accuracy is greater than $(1-\delta)$, the prevalence difference is at most $\delta$. This completes the proof.
 \end{proof}

Note that Theorem~\ref{thm:app} can also be used to guarantee that highly accurate predictions for a group $z$, implies a low value for $\Delta^z$. This leads to Corollary~\ref{cor:app}.

\begin{corollary}\label{cor:app}
If accuracy of a classifier for any sub-population $z$ is greater than $1-\delta$, then $\Delta^z\leq \delta$. 
\end{corollary}

\noindent The following theorem gives insight on why $\cape$ works. In the subsequent discussion, we drop the parameter $\Testdata$ from the notations $\hat{q}$ and $\rho$ and $\hat{\rho}$ since we exclusively refer to these values in the context of the test dataset $\Testdata$ only.
\begin{theorem} \label{claim:cape}
Let $\Theta = \{\frac{\epsilon}{2}, \frac{3\epsilon}{2}, \frac{5\epsilon}{2} \ldots, \left(k-\frac{1}{2}\right)\epsilon\}$ where $\epsilon\in(0,1)$ and $k=\left\lfloor\frac{1}{\epsilon}+\frac{1}{2}\right\rfloor$. For a group $z$, and test dataset $\Testdata$, let the quantifier be such that $|\rho^z - \hat{q}^z| \leq \delta_1$, and the classifiers be such that $|\theta_j - \hat{\rho}^z_{j}| \leq \delta_2$ for all $j\in\{1,\ldots,k\}$, for small $\delta_1$ and $\delta_2$. Then, for the best classifier $$J := \displaystyle \argmin_{j \in \{1, \ldots, k\}} |\predPrev^z_{j} - \hat{q}^z|,$$ the following holds:
$$|\rho^z - \hat{\rho}^z_{J}| \leq \delta_1 + \delta_2 + \frac{\epsilon}{2}.$$
\end{theorem}

\begin{proof}
For the best classifier $J$, the prevalence difference of a group $z$ can be upper bounded using triangle inequality:
\begin{eqnarray}
|\rho^z-\hat{\rho}^z_J|& \leq & |\rho^z-\hat{q}^z| + |\hat{q}^z - \hat{\rho}^z_J| \nonumber \\
& \leq & \delta_1 + |\hat{q}^z - \hat{\rho}^z_J| \label{eq:pdJ}
\end{eqnarray}
Inequality~(\ref{eq:pdJ}) is implied by the assumption on the quantifier's performance, i.e., $|\rho^z - \hat{q}^z| \leq \delta_1$. 
To provide an upper bound for $|\hat{q}^z - \hat{\rho}^z_J|$, we pick $J'$ such that $$J' = \displaystyle \argmin_{j \in \{1, \ldots, k\}} |\hat{q}^z - \theta_{J'}|,\quad\quad \mbox{ where } \theta_{J'}=\left(J'-\frac{1}{2}\right)\epsilon$$ 
Since $\hat{q}^z \in [0,1]$, it is at most $\epsilon/2$ away from one of the fractional values in $\{\frac{\epsilon}{2}, \frac{3\epsilon}{2}, \frac{5\epsilon}{2} \ldots, \left(k-\frac{1}{2}\right)\epsilon\}$. Therefore,
\begin{equation}
|\hat{q}^z - \theta_{J'}|\leq \epsilon/2\label{eq:q-theta}
\end{equation}
We use Inequality~(\ref{eq:q-theta}) to provide an upper bound to the expression $|\hat{q}^z - \hat{\rho}^z_J|$, using  case-by-case analysis.\\

\noindent Case~$1$: Assume $\hat{q}^z<\hat{\rho}^z_J$. This leaves us with three possibilities for the value of $\theta_{J'}$:
	\begin{enumerate}
	\item Assume $\theta_{J'}\geq\hat{\rho}^z_J$. Then, 
	\begin{eqnarray}
	\hat{\rho}^z_J- \hat{q}^z\leq \theta_{J'}- \hat{q}^z \leq \epsilon/2
	\end{eqnarray}
	\item Assume $\hat{q}^z\leq\theta_{J'}<\hat{\rho}^z_J$. Now, we bound the desired quantity using the value of $\hat{\rho}_{J'}$. Note that $|\hat{q}^z-\hat{\rho}_{J}|\leq |\hat{q}^z-\hat{\rho}_{J'}|$ since $J$ is the best classifier. Thus, either $\hat{\rho}_{J'}\geq \hat{\rho}_{J}$ or $\hat{\rho}_{J'}\leq \hat{q}^z$. 
		\begin{enumerate}
		\item Assume $\hat{\rho}_{J'}\leq\hat{q}^z$. Then, 
		\begin{eqnarray}
		|\hat{q}^z - \hat{\rho}^z_J| \leq  \hat{q}^z - \hat{\rho}^z_{J'} \leq  \theta_{J'} - \hat{\rho}^z_{J'}
		\leq \delta_2
		\end{eqnarray}
		
		\item Assume $\hat{\rho}_{J'}\geq\hat{\rho}_J$. Then, 
		\begin{eqnarray}
		|\hat{q}^z - \hat{\rho}^z_J|
		&\leq& (\hat{\rho}^z_{J'} - \hat{q}^z)\nonumber\\
		&=& (\hat{\rho}^z_{J'} - \theta_{J'}) + (\theta_{J'} - \hat{q}^z)\nonumber\\
		&\leq &\delta_2 + \epsilon/2
		\end{eqnarray}	
		\end{enumerate}
		\item Assume $\theta_{J'}<\hat{q}^z$. Now, we bound the desired quantity using the value of $\hat{\rho}_{J'}$, and there can be three cases. 
		\begin{enumerate}
		\item Assume $\hat{\rho}_{J'}\leq \theta_{J'}$. Then, 
		\begin{eqnarray}
		|\hat{q}^z - \hat{\rho}^z_J| &\leq & \hat{q}^z - \hat{\rho}^z_{J'}\nonumber\\
		&\leq &  (\hat{q}^z - \theta_{J'})+(\theta_{J'} - \hat{\rho}^z_{J'})\nonumber\\
		& \leq & \epsilon/2 + \delta_2
		\end{eqnarray}
		
		\item Assume $\theta_{J'}<\hat{\rho}_{J'}\leq\hat{q}^z$. Then, 
		\begin{eqnarray}
		|\hat{q}^z - \hat{\rho}^z_J|
		\leq \hat{q}^z - \hat{\rho}^z_{J'} 
		&\leq & \hat{q}^z - \theta_{J'}\nonumber\\
		&\leq &\epsilon/2
		\end{eqnarray}
		
		\item Assume $\hat{\rho}_{J'}>\hat{\rho}_{J}$. Then, 
		\begin{eqnarray}
		|\hat{q}^z - \hat{\rho}^z_J|
		\leq \hat{\rho}^z_{J'} -\theta_{J'}
		\leq \delta_2
		\end{eqnarray}
	\end{enumerate}
\end{enumerate}
Inequalities~$(10)$-$(15)$ establish the following upper bound when $\hat{q}^z< \hat{\rho}^z_J$,
\begin{equation}
|\hat{q}^z - \hat{\rho}^z_J| \leq \delta_2 + \epsilon/2.\label{eq:q-rho}
\end{equation}
Case $2$: $\hat{q}^z\geq \hat{\rho}^z_J$. An analysis analogous to Case $1$ gives the same inequality as (\ref{eq:q-rho}).
Combining Inequalities~(\ref{eq:pdJ}) and (\ref{eq:q-rho}), we obtain the desired upper bound of $\delta_1 + \delta_2 +\epsilon/2$ on the quantity $|\rho^z-\hat{\rho}^z|$.
\end{proof}

%
%

%
%
%
\section{Experimental Evaluation}
\label{sec:eval}

We first evaluate $\cape$ on synthetically generated datasets. We then compare it with other fair classifiers on the real-world COMPAS \cite{url:compas} and MEPS \cite{url:meps} datasets, where we observe possible prior-probability shifts. 
$\cape$ is open source but the link is retracted for anonymity.
The performance of $\cape$ on a wide range of fairness-metrics, across all these datasets, enforces our proposal that $\cape$ should be used for predictions under prior-probability shifts.

\subsection{Datasets}
\label{sec:datasets}

\noindent \underline{Synthetic}: We assume a generative model with $3$ features---sensitive attribute $Z \in \{0,1\}$, and two additional attributes $U$ and $V$---along with the label $Y\in \{0,1\}$. We assume that the overall population distribution is generated as 
$\mathcal{P}(U,V,Z,Y)=\mathcal{P}(U,V|Z,Y)\cdot \mathcal{P}(Z|Y)\cdot \mathcal{P}(Y)$. We further consider equal representation of the two population subgroups, i.e., $\mathcal{P}(Z$=$1|Y)=\mathcal{P}(Z$=$0|Y)$ for each $Y \in \{0,1\}$. $U$ and $V$ are conditionally independent: $\mathcal{P}(U,V|Z,Y)\!\!=\!\! \mathcal{P}(U,V|Y)\!\! =\!\! \mathcal{P}(U|Y)\cdot\mathcal{P}(V|Y)$, and the distributions are considered to be Gaussian ($\mathcal{N}$) with the following parameters: $\mathcal{P}(U|Y$=$1)\sim \mathcal{N}(15, 10)$, $\mathcal{P}(U|Y$=$0)\sim \mathcal{N}(5,5)$, $\mathcal{P}(V|Y$=$1)\sim \mathcal{N}(20, 10)$, and $\mathcal{P}(V|Y$=$0)\sim \mathcal{N}(40,10)$.

%

We generate $50000$ instances for the training dataset $D$ with equal label distribution, i.e., $\rho^z_D = 0.5$. However, while generating the test set, the prevalence parameters $\rho^z_{\Testdata}$ are different. We generated $81$ different types of test datasets, each obtained by varying the prevalences for both subgroups $z\in\{0,1\}$, such that $\rho_{\Testdata}^z \in\{0.1, \ldots, 0.9\}$.\\ 

\noindent \underline{{COMPAS}} dataset contains demographic information and criminal history for pre-trial defendants in Broward County, Florida. 
The goal of learning is to predict whether an individual re-offends. We consider $\mathtt{is\_recid}$ as $Y$ labels and $\mathtt{race}$ as the sensitive attribute ($Z$=$1$ denotes African-Americans, while $Z$=$0$ denotes Caucasians). We pre-processed the dataset to remove rows containing missing or invalid information. Our training dataset comprises $4278$ records whose screening dates were in the year $2013$ (of which $59.70\%$ are African-Americans), while the test dataset comprises $1809$ records screened in the year $2014$ (of which $60.86\%$ are African-Americans). 
\\

\begin{figure*}[hb]
    \centering
    \begin{subfigure}{.33\textwidth}
      \centering
      \includegraphics[width=0.9\textwidth]{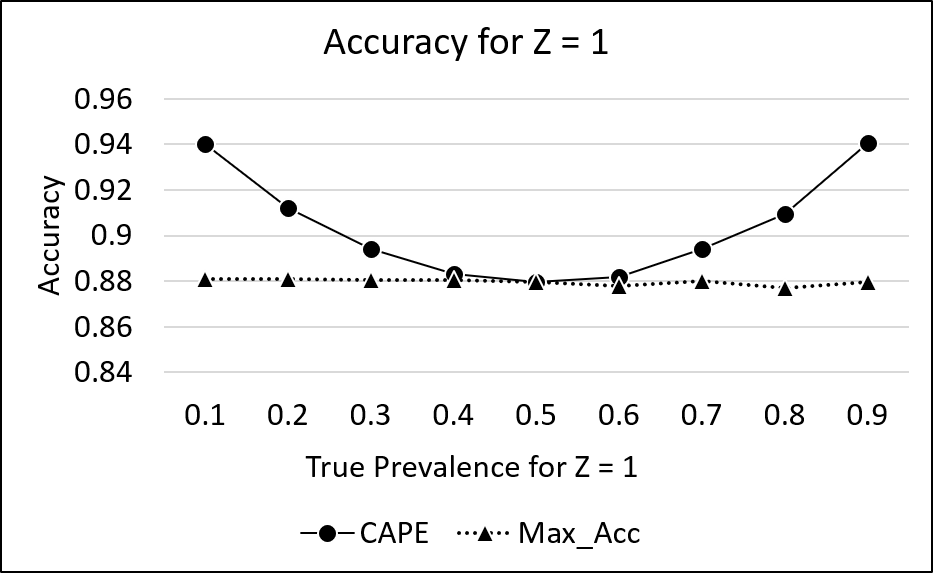}
      \captionof{figure}{Accuracy of $Z=1$.}
      \label{fig:synth-acc}
    \end{subfigure}%
    \begin{subfigure}{.33\textwidth}
      \centering
      \includegraphics[width=0.9\textwidth]{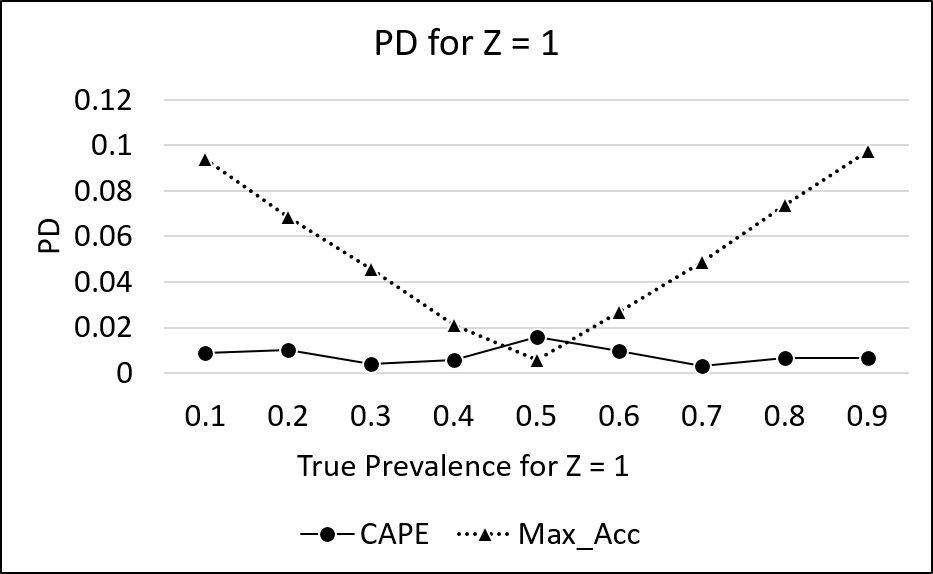}
      \captionof{figure}{Prevalence Difference for $Z=1$.}
      \label{fig:synth-pd}
    \end{subfigure}%
    \begin{subfigure}{.33\textwidth}
        \centering
        \includegraphics[width=0.9\textwidth]{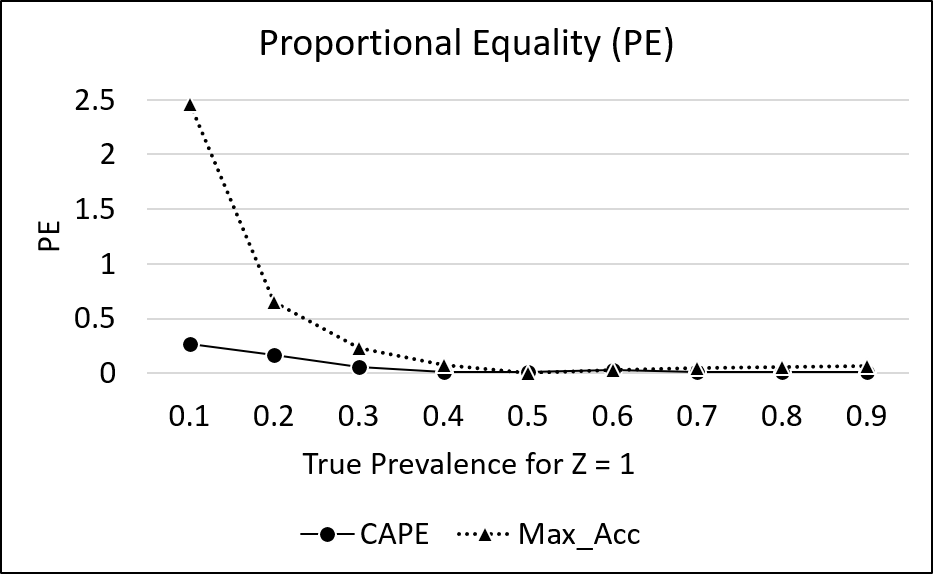}
        \captionof{figure}{Proportional Equality (PE$^{0,1}$).}
        \label{fig:synth-pe}
    \end{subfigure}
    \caption{\label{fig:synth-summary}Comparing accuracy, PD and PE metrics on synthetic test datasets with varying prevalences for group $Z$=$1$. The prevalence for group $Z$=$0$ is fixed at $0.5$. The reported results are averaged over $20$ iterations and the standard deviation is of the order $10^{-3}$.}
\end{figure*}

\noindent \underline{MEPS} comprises surveys carried out on individuals, health care professions, and employers in the United States. The feature $\mathtt{UTILIZATION}$ measures the total number of trips involved in availing some sort of medical facility. 
The classification task involves predicting whether $\mathtt{UTILIZATION} \ge 10$. We consider $\mathtt{RACE}$ as the sensitive attribute ($Z$=$1$ denotes `Non-Whites'). The surveys for the year $2015$ is our training set (with $33400$ data points, of which $62.86\%$ are `Non-Whites'), and the surveys for $2016$ is our test set (with $32006$ data points, of which $61.72\%$ are `Non-Whites'). 

\subsection{Other Algorithms for Comparison}
\label{sec:other-fair-algo}
We compare $\cape$ against an accuracy-maximizing classifier, \textbf{\texttt{Max\_Acc}}. It uses the same algorithm used by $\cape$ in the module $\clalgo$. On the real-world datasets, we additionally compare $\cape$ with the following (in-, pre- and post-processing) fair algorithms, implemented in the IBM AI Fairness 360~\cite{aif360-oct-2018} toolkit---Reweighing (\textbf{\texttt{Reweigh}})~\cite{kamiran2012data}, Adversarial Debiasing (\textbf{\texttt{AD}})~\cite{zhang2018mitigating}, and variants of \textbf{\texttt{Meta\_fair}}~\cite{celis2018classification}, Calibrated Equalized Odds Postprocessing (\textbf{\texttt{CEOP}})~\cite{pleiss2017fairness}, Reject Option Classification (\textbf{\texttt{ROC}})~\cite{kamiran2012decision}. 
These algorithms target fairness notions other than PE. We evaluate the extent to which these algorithms achieve PE fairness and compare how they perform on a set of other metrics (such as FPR-diff, FNR-diff, Accuracy-diff, and PD). While $\cape$ can handle multiple sensitive attributes, we choose one sensitive attribute for all the datasets to stay consistent with the implementation in the IBM AIF360 toolkit.

\subsection{Parameters and Modules used for $\cape$}
\label{sec:cape-params}
\begin{itemize}[leftmargin=*]
\item Prevalences: We set $\Theta=\{0.05, 0.15,\ldots,0.95\}$.
\item $\ppsampling$: As described in Section~\ref{sec:cape}.
\item $\qalgo$: \textit{Scaled Probability Average}~\cite{bella2010quantification}. 
\item $\clalgo$: As the synthetically generated datasets are created using simple generative models, we use generalized logistic regression ($\mathtt{glm}$) with regularization. For COMPAS and MEPS, we use gradient boosted algorithm ($\mathtt{gbm}$) and $10$-fold cross-validation for hyper-parameter tuning. 
\end{itemize}

\subsection{Results}
\noindent\textbf{Synthetic dataset}: We evaluated $\cape$ with $81$ types of test datasets, each with $\rho_\Testdata^z \in \{0.1, \ldots, 0.9\}$ for $z \in \{0,1\}$. The general trend we observe is that $\cape$ outperforms $\BClass$ whenever there is a significant shift in prior probabilities. We report two interesting sets of results here.

First, we consider test datasets with $\rho_{\Testdata}^0=0.5$, and $\rho_{\Testdata}^1$ ranging between $0.1$ and $0.9$. Figure~\ref{fig:synth-summary} summarizes our findings. Since $\cape$ accounts for prevalence changes, the accuracy of $\cape$ on $\Testdata$ for group $Z$=$1$ (Figure~\ref{fig:synth-acc}) is consistently higher than $\BClass$, except for the dataset with $\rho_\Testdata^1 = 0.5$ where the accuracies become nearly equal. 
The prevalence difference for $Z$=$1$ (Figure~\ref{fig:synth-pd}) is lower for $\cape$ whenever there is a prior probability shift (i.e., when $\rho_\Testdata^1 \neq 0.5$). In fact, for $\BClass$, $\predPrev_\Testdata^1$ remains $0.5$ across all the test datasets. Thus, $\Delta^1_{\Testdata}$ for $\BClass$ increases linearly as $\rho_\Testdata^1$ moves away from $0.5$. Lastly, the predictions of $\cape$ consistently exhibit a lower valuation for PE (Figure~\ref{fig:synth-pe}), compared to $\BClass$. This highlights that the predictions of $\cape$ are more fair, compared to the purely accuracy maximizing $\BClass$.

Second, in Table~\ref{table:synth-pps}, we report results for scenarios where \emph{both} $\rho_\Testdata^0$ and $\rho_\Testdata^1$ significantly deviate from their corresponding prevalences in the training set. The results are representative of the general trend we observed in the other test datasets---$\cape$ outperforms $\BClass$ on accuracy, PD and PE metrics.

\newcommand{\colspacing}{\hspace{1.8em}}
    \begin{table}[ht]
    \scriptsize
    \centering
    \begin{tabular}{clrrrrrr}
    \centering

& & \multicolumn{2}{c}{\textbf{Accuracy}} & \multicolumn{2}{c}{$\mathbf{\Delta}$} & \multicolumn{2}{c}{\textbf{PE}$^{0,1}$} \\
\cmidrule(lr){3-4}
\cmidrule(lr){5-6}
\cmidrule(lr){7-8}

\textbf{Z}
& $\rho_{\Testdata}^z$
& \multicolumn{1}{c}{$\cape$}
& \multicolumn{1}{c}{$\BClass$}
& \multicolumn{1}{c}{$\cape$}
& \multicolumn{1}{c}{$\BClass$}
& \multicolumn{1}{c}{$\cape$}
& \multicolumn{1}{c}{$\BClass$} \\ [0.3em]

\toprule

\multicolumn{1}{c}{0}
& \multicolumn{1}{c}{0.1}
& \multicolumn{1}{r}{0.940}
& \multicolumn{1}{r}{0.880}
& \multicolumn{1}{r}{0.009}
& \multicolumn{1}{r}{0.094}
& \multirow{2}{*}{0.050}
& \multirow{2}{*}{0.104}
\\

\multicolumn{1}{c}{1}
& \multicolumn{1}{c}{0.1}
& \multicolumn{1}{r}{0.930}
& \multicolumn{1}{r}{0.855}
& \multicolumn{1}{r}{0.016}
& \multicolumn{1}{r}{0.110}
& 
& 
\\[0.3em]

\midrule

\multicolumn{1}{c}{0}
& \multicolumn{1}{c}{0.2}
& \multicolumn{1}{r}{0.894}
& \multicolumn{1}{r}{0.855}
& \multicolumn{1}{r}{0.017}
& \multicolumn{1}{r}{0.084}
& \multirow{2}{*}{0.012}
& \multirow{2}{*}{0.140}
\\

\multicolumn{1}{c}{1}
& \multicolumn{1}{c}{0.8}
& \multicolumn{1}{r}{0.909}
& \multicolumn{1}{r}{0.877}
& \multicolumn{1}{r}{0.006}
& \multicolumn{1}{r}{0.074}
& 
& 
\\[0.3em]

\midrule 

\multicolumn{1}{c}{0}
& \multicolumn{1}{c}{0.9}
& \multicolumn{1}{r}{0.929}
& \multicolumn{1}{r}{0.851}
& \multicolumn{1}{r}{0.012}
& \multicolumn{1}{r}{0.120}
& \multirow{2}{*}{0.003}
& \multirow{2}{*}{0.028}
\\

\multicolumn{1}{c}{1}
& \multicolumn{1}{c}{0.9}
& \multicolumn{1}{r}{0.940}
& \multicolumn{1}{r}{0.879}
& \multicolumn{1}{r}{0.006}
& \multicolumn{1}{r}{0.097}
& 
&
\\[0.3em]

\bottomrule

    \end{tabular}
    \caption{\label{table:synth-pps}Accuracy, $\Delta$ and PE values on the synthetic datasets when test set $\Testdata$ is such that $\rho_\Testdata^z\neq0.5$, for both groups $z \in \{0, 1\}$.}
    \end{table}


\noindent\textbf{Real-world datasets}: 
For COMPAS, columns $3$ and $4$ of Table~\ref{table:quant} highlight that the true prevalences of the training (year $2013$) and test (year $2014$) datasets are significantly different. This is indicative of a possible prior probability shift. Column $5$ shows that the $\qalgo$ module of $\cape$ makes a good estimate of the true prevalences of the test dataset. 

    \begin{table}[h!]
        \scriptsize
        \centering
        \begin{tabular}{ccccc}
        \centering

& \textbf{\begin{tabular}[c]{@{}c@{}}$\ $\\ $ $\\ \textbf{Z} \\ \end{tabular}}
&  \textbf{\begin{tabular}[c]{@{}c@{}}Training Data\\ True Prevalence\\ ${\rho_{D}^z}$ \\ \end{tabular}}
& \textbf{\begin{tabular}[c]{@{}c@{}}Test Data\\ True Prevalence\\ ${\rho_{\Testdata}^z}$ \\ \end{tabular}}
& \textbf{\begin{tabular}[c]{@{}c@{}}Quantifier's\\ Estimate\\ ${\hat{q}^(\Testdata^z)}$ \\ \end{tabular}} \\
\toprule

\multirow{2}{*}{\textbf{COMPAS}}
& \multicolumn{1}{c}{\textbf{0}}
& \multicolumn{1}{c}{0.327}
& \multicolumn{1}{c}{0.636}
& \multicolumn{1}{c}{0.592}

\\

& \multicolumn{1}{c}{\textbf{1}}
& \multicolumn{1}{c}{0.486}
& \multicolumn{1}{c}{0.706}
& \multicolumn{1}{c}{0.644}

\\[0.3em]

\midrule

\multirow{2}{*}{\textbf{MEPS}}
& \multicolumn{1}{c}{\textbf{0}}
& \multicolumn{1}{c}{0.253}
& \multicolumn{1}{c}{0.253}
& \multicolumn{1}{c}{0.273}

\\

& \multicolumn{1}{c}{\textbf{1}}
& \multicolumn{1}{c}{0.124}
& \multicolumn{1}{c}{0.117}
& \multicolumn{1}{c}{0.123} \\[0.3em]

\bottomrule

        \end{tabular}
        \caption{Column $3$ and $4$ show possible prior probability shifts in COMPAS and MEPS. Column $5$ highlights the prevalence estimates obtained by $\qalgo$ module of $\cape$ on the test datasets. \label{table:quant}}
        \end{table}
        
\begin{table*}[ht]
\scriptsize
\centering
\begin{tabular}{cclrrrrrrrrrrrrrr}
\centering

& & & \multicolumn{3}{c}{\textbf{FPR}} & \multicolumn{3}{c}{\textbf{FNR}} & \multicolumn{3}{c}{\textbf{Accuracy}} & \multicolumn{2}{c}{\begin{tabular}{@{}c@{}}\textbf{Prediction} \\ \textbf{Prevalences}\end{tabular}} & \multicolumn{2}{c}{$\mathbf{\Delta}$}\\
    \cmidrule(lr){4-6}
    \cmidrule(lr){7-9}
    \cmidrule(lr){10-12}
    \cmidrule(lr){13-14}
    \cmidrule(lr){15-16}
    
    &
    & \textbf{Algorithms}
    & \multicolumn{1}{c}{\textbf{Z = 0}}
    & \multicolumn{1}{c}{\textbf{Z = 1}}
    & \multicolumn{1}{c}{\textbf{diff}}
    & \multicolumn{1}{c}{\textbf{Z = 0}}
    & \multicolumn{1}{c}{\textbf{Z = 1}}
    & \multicolumn{1}{c}{\textbf{diff}}
    & \multicolumn{1}{c}{\textbf{Z = 0}}
    & \multicolumn{1}{c}{\textbf{Z = 1}}
    & \multicolumn{1}{c}{\textbf{diff}}
    & \multicolumn{1}{c}{\textbf{Z = 0}}
    & \multicolumn{1}{c}{\textbf{Z = 1}}
    & \multicolumn{1}{c}{\textbf{Z = 0}}
    & \multicolumn{1}{c}{\textbf{Z = 1}}
    & \multicolumn{1}{c}{\textbf{PE}$^{0,1}$} \\[0.3em]
    
    \toprule
    
    \parbox[c]{2mm}{\multirow{15}{*}{\rotatebox[origin=c]{90}{\ \ \small \textbf{COMPAS}}}}
    & \parbox[c]{2mm}{\multirow{3}{*}{\rotatebox[origin=c]{90}{\ \ $\cape$}}}
    & $\cape$-$\Testdata$
    & \multicolumn{1}{r}{0.461}
    & \multicolumn{1}{r}{0.380}
    & \multicolumn{1}{r}{0.081}
    & \multicolumn{1}{r}{0.302}
    & \multicolumn{1}{r}{0.275}
    & \multicolumn{1}{r}{0.027}
    & \multicolumn{1}{r}{0.640}
    & \multicolumn{1}{r}{0.694}
    & \multicolumn{1}{r}{0.054}
    & \multicolumn{1}{r}{0.612}
    & \multicolumn{1}{r}{0.623}
    & \multicolumn{1}{r}{0.024}
    & \multicolumn{1}{r}{0.083}
    & \multicolumn{1}{r}{0.082}
    \\
    
    &
    & $\cape$-1
    & \multicolumn{1}{r}{0.271}
    & \multicolumn{1}{r}{0.290}
    & \multicolumn{1}{r}{0.019}
    & \multicolumn{1}{r}{0.451}
    & \multicolumn{1}{r}{0.322}
    & \multicolumn{1}{r}{0.129}
    & \multicolumn{1}{r}{0.614}
    & \multicolumn{1}{r}{0.687}
    & \multicolumn{1}{r}{0.073}
    & \multicolumn{1}{r}{0.448}
    & \multicolumn{1}{r}{0.564}
    & \multicolumn{1}{r}{0.188}
    & \multicolumn{1}{r}{0.142}
    & \multicolumn{1}{r}{0.119}
    \\[0.3em]
    
    
    \cmidrule(lr){2-17}
    
    &
    & $\BClass$
    & \multicolumn{1}{r}{0.132}
    & \multicolumn{1}{r}{0.259}
    & \multicolumn{1}{r}{0.127}
    & \multicolumn{1}{r}{0.629}
    & \multicolumn{1}{r}{0.340}
    & \multicolumn{1}{r}{0.289}
    & \multicolumn{1}{r}{0.552}
    & \multicolumn{1}{r}{0.684}
    & \multicolumn{1}{r}{0.132}
    & \multicolumn{1}{r}{0.284}
    & \multicolumn{1}{r}{0.542}
    & \multicolumn{1}{r}{0.352}
    & \multicolumn{1}{r}{0.163}
    & \multicolumn{1}{r}{0.376}
    \\[0.3em]
    
    \cmidrule(lr){2-17} 
    
    &\parbox[c]{2mm}{\multirow{2}{*}{\rotatebox[origin=c]{90}{\ \ {\textbf{Pre}}}}}
    
    & $\mathtt{Reweigh}$
    & \multicolumn{1}{r}{0.283}
    & \multicolumn{1}{r}{0.139}
    & \multicolumn{1}{r}{0.144}
    & \multicolumn{1}{r}{0.493}
    & \multicolumn{1}{r}{0.543}
    & \multicolumn{1}{r}{0.050}
    & \multicolumn{1}{r}{0.583}
    & \multicolumn{1}{r}{0.576}
    & \multicolumn{1}{r}{0.007}
    & \multicolumn{1}{r}{0.425}
    & \multicolumn{1}{r}{0.363}
    & \multicolumn{1}{r}{0.211}
    & \multicolumn{1}{r}{0.343}
    & \multicolumn{1}{r}{0.271} \\ [0.4em]
    
    \cmidrule(lr){2-17}
    
    & \parbox[c]{2mm}{\multirow{4}{*}{\rotatebox[origin=c]{90}{\ \ \textbf{In}}}}
    & $\mathtt{Meta}$-$\mathtt{fair}$-sr
    & \multicolumn{1}{r}{0.977}
    & \multicolumn{1}{r}{0.849}
    & \multicolumn{1}{r}{0.128}
    & \multicolumn{1}{r}{0.102}
    & \multicolumn{1}{r}{0.492}
    & \multicolumn{1}{r}{0.390}
    & \multicolumn{1}{r}{0.579}
    & \multicolumn{1}{r}{0.403}
    & \multicolumn{1}{r}{0.176}
    & \multicolumn{1}{r}{0.927}
    & \multicolumn{1}{r}{0.609}
    & \multicolumn{1}{r}{0.291}
    & \multicolumn{1}{r}{0.097}
    & \multicolumn{1}{r}{0.622}
    \\
    
    &
    & $\mathtt{Meta}$-$\mathtt{fair}$-fdr
    & \multicolumn{1}{r}{0.965}
    & \multicolumn{1}{r}{0.901}
    & \multicolumn{1}{r}{0.064}
    & \multicolumn{1}{r}{0.162}
    & \multicolumn{1}{r}{0.356}
    & \multicolumn{1}{r}{0.194}
    & \multicolumn{1}{r}{0.545}
    & \multicolumn{1}{r}{0.483}
    & \multicolumn{1}{r}{0.062}
    & \multicolumn{1}{r}{0.884}
    & \multicolumn{1}{r}{0.719}
    & \multicolumn{1}{r}{0.248}
    & \multicolumn{1}{r}{0.013}
    & \multicolumn{1}{r}{0.329}
    \\
    
    &
    & $\mathtt{AD}$
    & \multicolumn{1}{r}{0.124}
    & \multicolumn{1}{r}{0.167}
    & \multicolumn{1}{r}{0.043}
    & \multicolumn{1}{r}{0.638}
    & \multicolumn{1}{r}{0.467}
    & \multicolumn{1}{r}{0.171}
    & \multicolumn{1}{r}{0.549}
    & \multicolumn{1}{r}{0.621}
    & \multicolumn{1}{r}{0.072}
    & \multicolumn{1}{r}{0.275}
    & \multicolumn{1}{r}{0.425}
    & \multicolumn{1}{r}{0.361}
    & \multicolumn{1}{r}{0.281}
    & \multicolumn{1}{r}{0.253}
    \\
    [0.3em]
    
    \cmidrule(lr){2-17} 
    
    & \parbox[c]{2mm}{\multirow{7}{*}{\rotatebox[origin=c]{90}{\ \ \textbf{Post}}}}
    & $\mathtt{CEOP}$-fpr
    & \multicolumn{1}{r}{0.066}
    & \multicolumn{1}{r}{1.000}
    & \multicolumn{1}{r}{0.934}
    & \multicolumn{1}{r}{0.722}
    & \multicolumn{1}{r}{0.000}
    & \multicolumn{1}{r}{0.722}
    & \multicolumn{1}{r}{0.517}
    & \multicolumn{1}{r}{0.706}
    & \multicolumn{1}{r}{0.189}
    & \multicolumn{1}{r}{0.201}
    & \multicolumn{1}{r}{1.000}
    & \multicolumn{1}{r}{0.435}
    & \multicolumn{1}{r}{0.294}
    & \multicolumn{1}{r}{0.699}
    \\
    
    &
    & $\mathtt{CEOP}$-fnr
    & \multicolumn{1}{r}{0.000}
    & \multicolumn{1}{r}{0.247}
    & \multicolumn{1}{r}{0.247}
    & \multicolumn{1}{r}{1.000}
    & \multicolumn{1}{r}{0.390}
    & \multicolumn{1}{r}{0.610}
    & \multicolumn{1}{r}{0.364}
    & \multicolumn{1}{r}{0.652}
    & \multicolumn{1}{r}{0.288}
    & \multicolumn{1}{r}{0.000}
    & \multicolumn{1}{r}{0.503}
    & \multicolumn{1}{r}{0.636}
    & \multicolumn{1}{r}{0.203}
    & \multicolumn{1}{r}{0.900}
    \\
    
    &
    & $\mathtt{CEOP}$-weighted
    & \multicolumn{1}{r}{0.000}
    & \multicolumn{1}{r}{0.194}
    & \multicolumn{1}{r}{0.194}
    & \multicolumn{1}{r}{1.000}
    & \multicolumn{1}{r}{0.405}
    & \multicolumn{1}{r}{0.495}
    & \multicolumn{1}{r}{0.364}
    & \multicolumn{1}{r}{0.657}
    & \multicolumn{1}{r}{0.292}
    & \multicolumn{1}{r}{0.000}
    & \multicolumn{1}{r}{0.477}
    & \multicolumn{1}{r}{0.636}
    & \multicolumn{1}{r}{0.229}
    & \multicolumn{1}{r}{0.900}
    \\
    
    
    &
    & $\mathtt{ROC}$-aod
    & \multicolumn{1}{r}{0.004}
    & \multicolumn{1}{r}{0.019}
    & \multicolumn{1}{r}{0.015}
    & \multicolumn{1}{r}{0.978}
    & \multicolumn{1}{r}{0.900}
    & \multicolumn{1}{r}{0.078}
    & \multicolumn{1}{r}{0.377}
    & \multicolumn{1}{r}{0.360}
    & \multicolumn{1}{r}{0.017}
    & \multicolumn{1}{r}{0.016}
    & \multicolumn{1}{r}{0.076}
    & \multicolumn{1}{r}{0.620}
    & \multicolumn{1}{r}{0.630}
    & \multicolumn{1}{r}{0.879}
    \\
    
    &
    & $\mathtt{ROC}$-eod
    & \multicolumn{1}{r}{0.019}
    & \multicolumn{1}{r}{0.046}
    & \multicolumn{1}{r}{0.027}
    & \multicolumn{1}{r}{0.911}
    & \multicolumn{1}{r}{0.782}
    & \multicolumn{1}{r}{0.129}
    & \multicolumn{1}{r}{0.414}
    & \multicolumn{1}{r}{0.434}
    & \multicolumn{1}{r}{0.020}
    & \multicolumn{1}{r}{0.064}
    & \multicolumn{1}{r}{0.167}
    & \multicolumn{1}{r}{0.572}
    & \multicolumn{1}{r}{0.539}
    & \multicolumn{1}{r}{0.517}
    \\ [0.3em]

    \bottomrule
    \\
    \parbox[c]{2mm}{\multirow{15}{*}{\rotatebox[origin=c]{90}{\ \ \small \textbf{MEPS}}}}
    & \parbox[c]{2mm}{\multirow{3}{*}{\rotatebox[origin=c]{90}{\ \ $\cape$}}}
    & $\cape$-$\Testdata$
    & \multicolumn{1}{r}{0.131}
    & \multicolumn{1}{r}{0.068}
    & \multicolumn{1}{r}{0.063}
    & \multicolumn{1}{r}{0.425}
    & \multicolumn{1}{r}{0.488}
    & \multicolumn{1}{r}{0.063}
    & \multicolumn{1}{r}{0.794}
    & \multicolumn{1}{r}{0.883}
    & \multicolumn{1}{r}{0.089}
    & \multicolumn{1}{r}{0.243}
    & \multicolumn{1}{r}{0.120}
    & \multicolumn{1}{r}{0.010}
    & \multicolumn{1}{r}{0.003}
    & \multicolumn{1}{r}{0.135}
    \\

    &
    & $\cape$-1
    & \multicolumn{1}{r}{0.175}
    & \multicolumn{1}{r}{0.087}
    & \multicolumn{1}{r}{0.088}
    & \multicolumn{1}{r}{0.347}
    & \multicolumn{1}{r}{0.423}
    & \multicolumn{1}{r}{0.076}
    & \multicolumn{1}{r}{0.781}
    & \multicolumn{1}{r}{0.874}
    & \multicolumn{1}{r}{0.093}
    & \multicolumn{1}{r}{0.296}
    & \multicolumn{1}{r}{0.144}
    & \multicolumn{1}{r}{0.043}
    & \multicolumn{1}{r}{0.027}
    & \multicolumn{1}{r}{0.049}
    \\[0.3em]

    
    \cmidrule(lr){2-17}
    
    &
    & $\BClass$
    & \multicolumn{1}{r}{0.004}
    & \multicolumn{1}{r}{0.012}
    & \multicolumn{1}{r}{0.008}
    & \multicolumn{1}{r}{0.910}
    & \multicolumn{1}{r}{0.888}
    & \multicolumn{1}{r}{0.022}
    & \multicolumn{1}{r}{0.766}
    & \multicolumn{1}{r}{0.890}
    & \multicolumn{1}{r}{0.124}
    & \multicolumn{1}{r}{0.037}
    & \multicolumn{1}{r}{0.014}
    & \multicolumn{1}{r}{0.216}
    & \multicolumn{1}{r}{0.103}
    & \multicolumn{1}{r}{0.483}
    \\[0.3em]
    
    \cmidrule(lr){2-17} 
  
    &\parbox[c]{2mm}{\multirow{2}{*}{\rotatebox[origin=c]{90}{\ \ {\textbf{Pre}}}}}
    
    & $\mathtt{Reweigh}$
    & \multicolumn{1}{r}{0.276}
    & \multicolumn{1}{r}{0.242}
    & \multicolumn{1}{r}{0.034}
    & \multicolumn{1}{r}{0.250}
    & \multicolumn{1}{r}{0.226}
    & \multicolumn{1}{r}{0.024}
    & \multicolumn{1}{r}{0.731}
    & \multicolumn{1}{r}{0.760}
    & \multicolumn{1}{r}{0.029}
    & \multicolumn{1}{r}{0.396}
    & \multicolumn{1}{r}{0.305}
    & \multicolumn{1}{r}{0.143}
    & \multicolumn{1}{r}{0.188}
    & \multicolumn{1}{r}{0.862} \\ [0.4em]
    
    \cmidrule(lr){2-17}

    &\parbox[c]{2mm}{\multirow{4}{*}{\rotatebox[origin=c]{90}{\ \ \textbf{In}}}}
    & $\mathtt{Meta}$-$\mathtt{fair}$-sr
    & \multicolumn{1}{r}{0.322}
    & \multicolumn{1}{r}{0.213}
    & \multicolumn{1}{r}{0.109}
    & \multicolumn{1}{r}{0.210}
    & \multicolumn{1}{r}{0.243}
    & \multicolumn{1}{r}{0.033}
    & \multicolumn{1}{r}{0.706}
    & \multicolumn{1}{r}{0.783}
    & \multicolumn{1}{r}{0.077}
    & \multicolumn{1}{r}{0.440}
    & \multicolumn{1}{r}{0.277}
    & \multicolumn{1}{r}{0.187}
    & \multicolumn{1}{r}{0.160}
    & \multicolumn{1}{r}{0.572}
    \\
    
    &
    & $\mathtt{Meta}$-$\mathtt{fair}$-fdr
    & \multicolumn{1}{r}{0.347}
    & \multicolumn{1}{r}{0.254}
    & \multicolumn{1}{r}{0.102}
    & \multicolumn{1}{r}{0.193}
    & \multicolumn{1}{r}{0.218}
    & \multicolumn{1}{r}{0.025}
    & \multicolumn{1}{r}{0.692}
    & \multicolumn{1}{r}{0.758}
    & \multicolumn{1}{r}{0.066}
    & \multicolumn{1}{r}{0.463}
    & \multicolumn{1}{r}{0.308}
    & \multicolumn{1}{r}{0.210}
    & \multicolumn{1}{r}{0.191}
    & \multicolumn{1}{r}{0.657}
    \\
    
    &
    & $\mathtt{AD}$
    & \multicolumn{1}{r}{0.062}
    & \multicolumn{1}{r}{0.051}
    & \multicolumn{1}{r}{0.011}
    & \multicolumn{1}{r}{0.644}
    & \multicolumn{1}{r}{0.569}
    & \multicolumn{1}{r}{0.075}
    & \multicolumn{1}{r}{0.791}
    & \multicolumn{1}{r}{0.889}
    & \multicolumn{1}{r}{0.098}
    & \multicolumn{1}{r}{0.136}
    & \multicolumn{1}{r}{0.095}
    & \multicolumn{1}{r}{0.117}
    & \multicolumn{1}{r}{0.022}
    & \multicolumn{1}{r}{0.728}
    \\
    [0.3em]
    
    \cmidrule(lr){2-17}
    
    &\parbox[c]{2mm}{\multirow{7}{*}{\rotatebox[origin=c]{90}{\ \ \textbf{Post}}}}
    & $\mathtt{CEOP}$-fpr
    & \multicolumn{1}{r}{0.078}
    & \multicolumn{1}{r}{0.000}
    & \multicolumn{1}{r}{0.078}
    & \multicolumn{1}{r}{0.573}
    & \multicolumn{1}{r}{1.000}
    & \multicolumn{1}{r}{0.427}
    & \multicolumn{1}{r}{0.797}
    & \multicolumn{1}{r}{0.883}
    & \multicolumn{1}{r}{0.086}
    & \multicolumn{1}{r}{0.166}
    & \multicolumn{1}{r}{0.000}
    & \multicolumn{1}{r}{0.087}
    & \multicolumn{1}{r}{0.117}
    & \multicolumn{1}{r}{$\mathtt{undef}$}
    \\
    
    &
    & $\mathtt{CEOP}$-fnr
    & \multicolumn{1}{r}{0.034}
    & \multicolumn{1}{r}{0.022}
    & \multicolumn{1}{r}{0.012}
    & \multicolumn{1}{r}{0.803}
    & \multicolumn{1}{r}{0.704}
    & \multicolumn{1}{r}{0.102}
    & \multicolumn{1}{r}{0.771}
    & \multicolumn{1}{r}{0.899}
    & \multicolumn{1}{r}{0.128}
    & \multicolumn{1}{r}{0.075}
    & \multicolumn{1}{r}{0.054}
    & \multicolumn{1}{r}{0.178}
    & \multicolumn{1}{r}{0.063}
    & \multicolumn{1}{r}{0.771}
    \\
    
    &
    & $\mathtt{CEOP}$-weighted
    & \multicolumn{1}{r}{0.032}
    & \multicolumn{1}{r}{0.021}
    & \multicolumn{1}{r}{0.011}
    & \multicolumn{1}{r}{0.816}
    & \multicolumn{1}{r}{0.704}
    & \multicolumn{1}{r}{0.112}
    & \multicolumn{1}{r}{0.770}
    & \multicolumn{1}{r}{0.899}
    & \multicolumn{1}{r}{0.129}
    & \multicolumn{1}{r}{0.070}
    & \multicolumn{1}{r}{0.053}
    & \multicolumn{1}{r}{0.183}
    & \multicolumn{1}{r}{0.064}
    & \multicolumn{1}{r}{0.839}
    \\
    
    &
    & $\mathtt{ROC}$-spd
    & \multicolumn{1}{r}{0.233}
    & \multicolumn{1}{r}{0.220}
    & \multicolumn{1}{r}{0.013}
    & \multicolumn{1}{r}{0.284}
    & \multicolumn{1}{r}{0.243}
    & \multicolumn{1}{r}{0.041}
    & \multicolumn{1}{r}{0.754}
    & \multicolumn{1}{r}{0.777}
    & \multicolumn{1}{r}{0.023}
    & \multicolumn{1}{r}{0.355}
    & \multicolumn{1}{r}{0.283}
    & \multicolumn{1}{r}{0.102}
    & \multicolumn{1}{r}{0.166}
    & \multicolumn{1}{r}{0.906}
    \\
    
    &
    & $\mathtt{ROC}$-aod
    & \multicolumn{1}{r}{0.329}
    & \multicolumn{1}{r}{0.253}
    & \multicolumn{1}{r}{0.076}
    & \multicolumn{1}{r}{0.205}
    & \multicolumn{1}{r}{0.210}
    & \multicolumn{1}{r}{0.005}
    & \multicolumn{1}{r}{0.702}
    & \multicolumn{1}{r}{0.752}
    & \multicolumn{1}{r}{0.050}
    & \multicolumn{1}{r}{0.447}
    & \multicolumn{1}{r}{0.216}
    & \multicolumn{1}{r}{0.194}
    & \multicolumn{1}{r}{0.199}
    & \multicolumn{1}{r}{0.745}
    \\
    
    &
    & $\mathtt{ROC}$-eod
    & \multicolumn{1}{r}{0.336}
    & \multicolumn{1}{r}{0.233}
    & \multicolumn{1}{r}{0.103}
    & \multicolumn{1}{r}{0.194}
    & \multicolumn{1}{r}{0.227}
    & \multicolumn{1}{r}{0.033}
    & \multicolumn{1}{r}{0.700}
    & \multicolumn{1}{r}{0.768}
    & \multicolumn{1}{r}{0.068}
    & \multicolumn{1}{r}{0.455}
    & \multicolumn{1}{r}{0.296}
    & \multicolumn{1}{r}{0.202}
    & \multicolumn{1}{r}{0.179}
    & \multicolumn{1}{r}{0.623}
    \\
    
    \bottomrule
    
\end{tabular}
\caption{\label{table:results}Comparing $\cape$ with $\BClass$ and other fair classifiers on the COMPAS and MEPS test datasets.}
\end{table*}

For MEPS, we observe a shift only for the group $Z$=$1$, between the training set (surveys in the year $2015$) and test set (surveys in $2016$). Since the differences in prevalences are rather small, this dataset is of interest---it allows us to investigate the performance of $\cape$ when the extent of prior probability shift is small. Though the prevalences estimated by $\qalgo$ seem similar to the training set, the difference in the estimates of $\qalgo$ and the prevalences of the test datasets are only $0.02$ and $0.006$, for $Z$=$0$ and $Z$=$1$ respectively, and are thus good estimates.
      
Table~\ref{table:results} summarizes the results on COMPAS and MEPS datasets for $\cape$, $\BClass$, and the other fair algorithms described in Section~\ref{sec:other-fair-algo}. Due to lack of space, we elaborate upon the results of the COMPAS dataset only. 

$\cape$-$\Testdata$ considers the whole test dataset $\Testdata$ during prediction, while $\cape$-$1$ considers individual instances during prediction (similar to what the other algorithms do). We expect $\cape$-$\Testdata$ to perform better than $\cape$-$1$ since the $\qalgo$ module is expected to perform better for larger test datasets. 

$\cape$-$\Testdata$ outperforms $\BClass$ on $\Delta$, and all the other fairness metrics (FPR-diff, FNR-diff, Accuracy-diff, and PE). The prediction prevalences of $\BClass$ ($0.284$ and $0.542$) are close to the true prevalences of the \emph{training} set ($0.327$ and $0.486$), which highlights the inability of $\BClass$ to account for the prior probability shift. One critical observation about $\cape$-$\Testdata$ is that FPR-diff=$0.081$ and FNR-diff=$0.027$ which implies that the predictions exhibit equalized odds. In comparison, these differences for $\BClass$ are $0.127$ and $0.289$. In fact, for $\BClass$,  $\mathrm{FPR}_{Z=1}$ is almost twice than $\mathrm{FPR}_{Z=0}$, whereas $\mathrm{FNR}_{Z=1}$ is almost half of $\mathrm{FNR}_{Z=0}$. This implies that $\BClass$ imposes unfair higher risks of recidivism on African-American defendants, while Caucasian defendants are predicted to have lower risks than they actually do.

The true prevalences of the two subgroups in the test dataset are close to each other (namely, $0.636$ and $0.706$). Thus, a classifier aiming to achieve statistical parity is expected to do well on PE. We observe this in Table~\ref{table:results}, where $\mathtt{ROC}$-spd (\underline{s}tatistical \underline{p}arity \underline{d}ifference) has lowest PE ($0.056$). However, its false positive rates are more than $0.9$ for both subgroups, which is unfair and harmful for both subgroups. This unfairness is also captured by the high $\Delta$ value of $\mathtt{ROC}$-spd. We observe that $\Delta^{0}$ is the lowest for $\cape$-$\Testdata$ among all other classifiers. For $Z=1$, $\mathtt{Meta\_fair}$-fdr (\underline{f}alse \underline{d}iscovery \underline{r}ate with the group fairness trade-off parameter $\tau$ set to $0.8$), is the only other fair classifier with a lower $\Delta^{1}$ value. However, the predictions of $\mathtt{Meta\_fair}$-fdr have high false positive rates, and low accuracies.



Note that a trivial classifier, which always predicts positive labels, will have FNR-diff=$0$, FPR-diff=$0$, Accuracy-diff=$0.07$. 
However, this classifier will have high PD for both groups ($\Delta^0$=$0.364$ and $\Delta^1$=$0.294$), which indicates a substantial skew between the false positives and false negatives. Thus, PD is an important metric that, in addition to accuracy, captures the learning ability of the classifiers.

We make a final observation on our experimental results. Since both COMPAS and MEPS are real-world datasets, the distributional changes highlighted in Table~\ref{table:quant} may not be due to prior probability shifts alone. Although $\cape$ is designed to handle only prior probability shifts, the good performance of both $\cape$-$\Testdata$ and $\cape$-$1$ on a wide range of metrics for these real-world datasets shows the robustness of our approach. 



A possible extension of $\cape$ includes handling other distributional changes, such as \emph{concept drifts}, that is, when $\mathcal{P}(X|Y,Z)$ changes but $\mathcal{P}(Y|Z)$ remains same.

\subsection{Acknowledgment}
Arpita Biswas gratefully acknowledges the support of a Google PhD Fellowship Award. 

\scriptsize
\bibliographystyle{named}
\bibliography{ijcai20bib}

\end{document}